\documentclass[10pt, twocolumn, twoside]{IEEEtran}

\usepackage{multirow}
\usepackage{makecell}

\usepackage{amsthm}
\usepackage{pdfpages}
\usepackage{textcomp}
\usepackage{epsfig,latexsym}
\usepackage{float}
\usepackage{indentfirst}
\usepackage{amsmath}
\usepackage{amssymb}
\usepackage{xcolor}

\usepackage{times}
\usepackage{subfigure}
\usepackage{psfrag}
\usepackage{cite}
\usepackage{flushend}  
\usepackage{epstopdf}
\usepackage{fancyhdr}
\usepackage{stfloats}
\usepackage{color}
\usepackage[noend]{algpseudocode}
\usepackage{hyperref}
\usepackage{algorithmicx,algorithm}
\usepackage{tabularx}
\newcommand{\E}{\mathrm{E}}

\newtheorem{lemma}{Lemma}

\newtheorem{proposition}{Proposition}

\begin{document}

\title{MIM-GAN-based Anomaly Detection for Multivariate Time Series Data}
\vspace{1.5em}
\author{\normalsize{\IEEEauthorblockN{Shan Lu\IEEEauthorrefmark{1}, Zhicheng Dong\IEEEauthorrefmark{1}\IEEEauthorrefmark{2}, Donghong Cai\IEEEauthorrefmark{3}, Fang Fang\IEEEauthorrefmark{4}, and Dongcai Zhao\IEEEauthorrefmark{1}}
}

\IEEEauthorblockA{\IEEEauthorrefmark{1}College of Information Science Technology, Tibet University, Lhasa 850000, China.}

\IEEEauthorblockA{\IEEEauthorrefmark{3}College of Information Science and Technology, Jinan University, Guangzhou 510632, China.}

\IEEEauthorblockA{\IEEEauthorrefmark{4}Department of Electrical and Computer Engineering, Western University, London N6A5B9, Canada.}

\IEEEauthorblockA{\IEEEauthorrefmark{2}Tibet Autonomous Region Plateau Complex Environmental Information Transmission and Processing Engineering Research Center; Advanced Technology Research Center of Tibet University, Lhasa 850000, China.}

\IEEEauthorblockA{Zhicheng Dong is the corresponding author of this paper.}

\IEEEauthorblockA{Email: explorerlu@utibet.edu.cn;  dongzc666@163.com; dhcai@jnu.edu.cn; fang.fang@uwo.ca; zdc@utibet.edu.cn.  }



}

%



\maketitle
\vspace{-2.0em}

\begin{abstract}

The loss function of Generative adversarial network (GAN) is an important factor that affects the quality and diversity of the generated samples for anomaly detection. In this paper, we propose an unsupervised multiple time series anomaly detection algorithm based on the GAN with message importance measure (MIM-GAN). In particular, the time series data is divided into subsequences using a sliding window. Then a generator and a discriminator designed based on the Long Short-Term Memory (LSTM) are employed to capture the temporal correlations of the time series data. To avoid the local optimal solution of loss function and the model collapse, we introduce an exponential information measure into the loss function of GAN. Additionally, a discriminant-reconstruction score is composed of discrimination and reconstruction loss. The global optimal solution for the loss function is derived and the  model collapse is proved to be avoided in our proposed MIM-GAN-based anomaly detection algorithm. Experimental results show that the proposed MIM-GAN-based anomaly detection algorithm has superior performance in terms of precision, recall, and F1-score.

\end{abstract}
 \vspace{-0.5em}

 \vspace{-0.5em}
\section{Introduction}
\vspace{0.5em}

 The Internet of Things (IoT)\cite{8920092} has been developing rapidly
and widely applied in various Cyber-Physical Systems (CPSs)\cite{7911887}, such as environmental supervision, smart factory, smart agriculture, data centers, etc.. These intelligent platforms usually employ a large number of sensors to monitor the operating status of devices by processing a large amount of time-series data. Especially, many of these CPSs are designed for critical missions. Consequently, the CPSs are prime targets of cyber attacks. Data anomaly detection plays a crucial role in safeguarding CPSs.

However, multiple challenges arise in the anomaly detection in multivariate time series data. For instance, traditional threshold-based anomaly detection methods can only handle data that is highly correlated and follows a multivariate Gaussian distribution\cite{9796836}, but it is difficult to process the multiple time series due to multiple variables, multiple features, and time-space correlation. Supervised machine learning methods\cite{9760098} cannot use a large amount of time series because this sequences are lack of labeled data. Unsupervised machine learning methods\cite{9179547}-\cite{2306789} have not only fully utilized the temporal, spatial correlations, but also the dependencies between multivariate data to detect abnormal data.

Recently, Choi Yeji et al.\cite{9070362} proposed  a novel Generative adversarial network (GAN) based anomaly detection and a localization framework for power plant data, which can transform multivariate time series into two-dimensional (2D) images, and use a GAN model to learn the mapping from a series of distance images to the next distance image. The encoder and decoder structure of the GAN model can capture both the temporal information and the correlation between different variables. However, the GAN-based anomaly detection may suffer from mode collapse or instability issues, which can affect the quality of the generated images and the anomaly detection performance. Furthermore, Martin Arjovsky et al. \cite{8825555} introduced Wasserstein GAN (W-GAN) as a solution to address the challenges encountered in GANs, including model instability, slow convergence, and mode collapse. The W-GAN employs Wasserstein distance as the loss function, which can avoid gradient vanishing, and improve the stability as well the convergence. However, the W-GAN increases the computational overhead, and requires tuning the gradient penalty coefficient.
Jiaxuan You et al.\cite{you2020stad} proposed GAN-based spatio-temporal anomaly detection (STAD-GAN), which uses a convolutional neural network to extract features. In addition, STAD-GAN requires a large amount of time and capacity to train GAN and neural network classifier, which may limit its applicability in real-world environments. During the training process, multiple hyperparameters need to be adjusted to balance the competitive relationship between the generator and discriminator.

To overcome the limitations of the existing research methods, we propose an anomaly detection algorithm based on the GAN with message importance measure(MIM-GAN), which uses a exponential function instead of a logarithmic function. Specifically, the conventional cross-entropy loss based on binary classification can lead to gradient vanishing in the generator and impede model training. We introduce the loss function in exponential form\cite{9405665}. However, the generator and discriminator cannot engage in adversarial training using the Adam optimizer. Hence, we replace the generator’s optimizer of MIM-GAN by AdamW algorithm. The global optimal solution of loss function in MIM-GAN is derived and the mode collapse is proved to be avoided. Moreover, a Discriminant Reconstruction-Score (DIRE-score) consisting of the discrimination loss and reconstruction loss is considered. The  MIM-GAN-based algorithm for anomaly detection can improve the diversity and richness of the generated samples. The experimental results represent that MIM-GAN-based algorithm for anomaly detection has superior performance compared to the existing GAN-based anomaly detection methods.
\begin{figure}[tp]
	\centering
	\vspace*{0.5em}
	\includegraphics[width=0.89\linewidth]{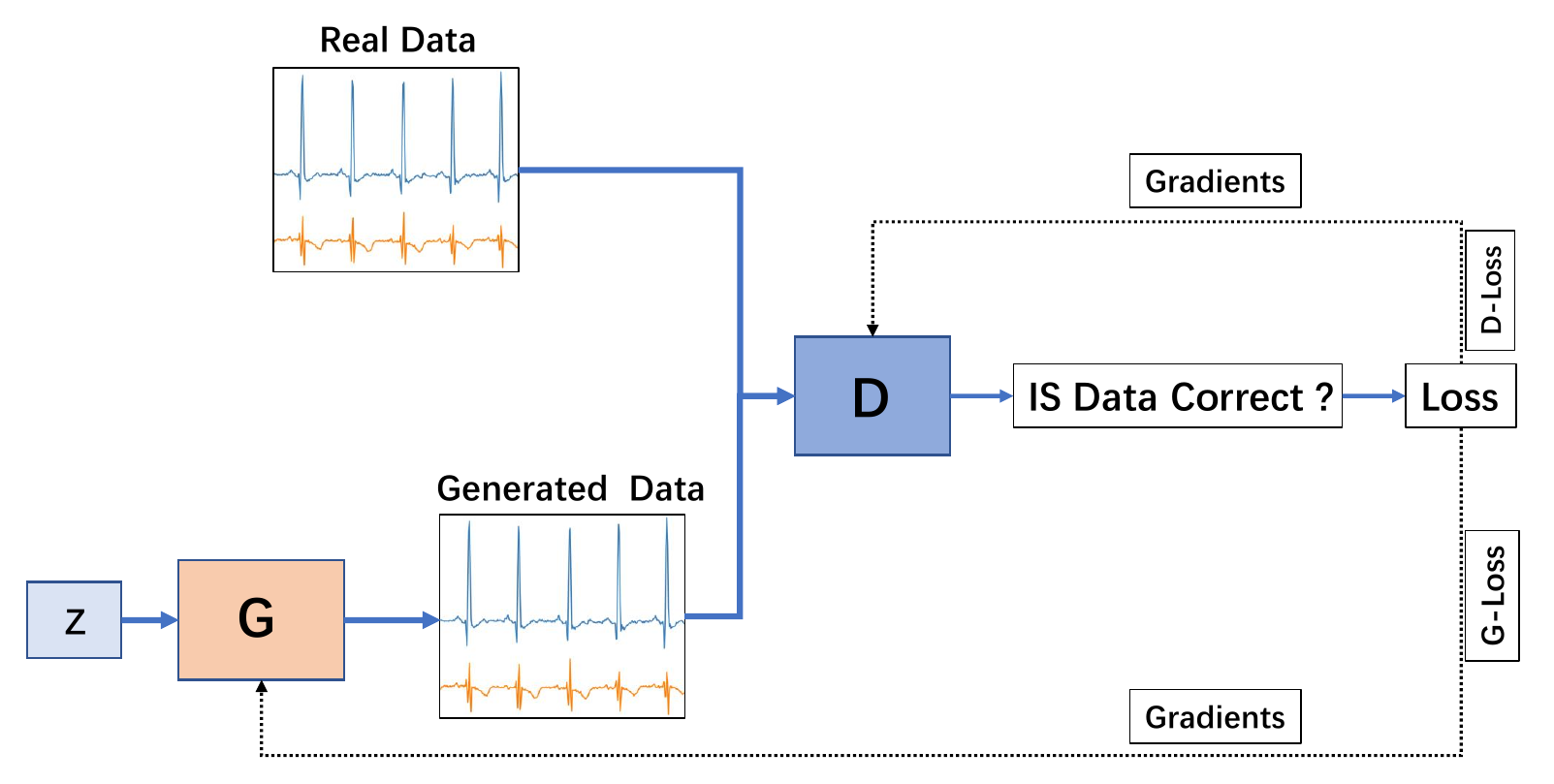}
	\vspace*{-0.5em}
	\caption{Illustration of anomaly detection for time series data based on GAN.}\label{abfl11}
\end{figure}
\section{Problem Formulation}
As we can see in Fig. \ref{abfl11}, the GAN network for anomaly detection of multivariate time series data, which consists of two neural networks: G and D. G takes noise $z$ as input and tries to generate fake data that resembles the distribution of the training data. D tries to determine whether the generated time series are real or fake. G is for the purpose of to make higher failure rate of the discriminator. The GAN network can be designed based on the following optimization problem:
\begin{align}
\min\limits_{\mathrm{D}}\space\max\limits_{\mathrm{G}}\space \mathcal{L}_{{KL\bf}}(\mathrm{D},\mathrm{G}),
\end{align}
where the objective function $\mathcal{L}_{KL}(D, G)$ is defined as
\begin{align}\label{esjdk1}
\mathcal{L}_{\mathrm{KL}}(D,\! G)&\!=\!E_{x \sim P_{data}}\![\log (D(x))]\!+\! E_{z \sim P_{G}}[\log[1\!-\!D(G(z))] \nonumber \\
\!\!&\!\!\!\!\!\!\!\!\!\!\!\!\!\!\!\!=E_{x \sim P_{data}}[\log (D(x))]\!+\!E_{x \sim P_{G}}[\log (1-D(x))],
\end{align}
where $P_{data}$ denotes the distribution of real data and $P_{G}$ is the distribution for generated data.

The GAN is employed for anomaly detection of multiple time series data, because it can learn the distribution from a large number of unlabeled normal data, and identify the abnormal data by reconstruction error or discrimination error. However, the GAN-based anomaly detection still has the following challenges:
\begin{itemize}
  \item The training of GANs is known for its instability, often resulting in model collapse or failure to converge during training. Consequently, both the generator and discriminator may struggle to effectively capture the diversity and complexity of the data. Model collapse, in particular, poses a significant challenge in GAN generation due to its inherent characteristics, such as high dimensionality, nonlinearity, and dynamics. Successfully addressing these challenges requires GANs to possess robust modeling and generalization capabilities.
    \item
The presence of local optimal solutions for problem \eqref{esjdk1} severely restricts the application scope and scenarios of GAN-based generation for multivariate time series data. This limitation becomes especially troublesome in application domains that require diverse and high-quality data. GANs are expected to generate multivariate time series data that exhibits both high quality and diversity.
  \item The design of the GAN loss function contributes to its training difficulties, often resulting in insufficient diversity in the generated samples. Traditional adversarial networks utilize binary cross-entropy loss may further exacerbate the problem by causing the generator gradient to vanish, making the model challenging to train effectively.
\end{itemize}

To address the limitations of the aforementioned research methods and to accommodate the temporal relationship and dependencies among variables, we propose an innovative algorithm for anomaly detection in multivariate time series data. Our approach, termed MIM-GAN, aims to overcome the challenges encountered in training traditional GAN models. By incorporating information importance measurement, the proposed MIM-GAN anomaly detection offers an effective solution to enhance the detection performance and capture the variety inherent in multivariate time series data.

\section{ Anomaly Detection for Multivariate Time Based on MIM-GAN}
In this section, we propose an MIM-GAN-based anomaly detection algorithm, which consists of MIM-GAN network training and MIM-GAN-based anomaly detection. For the network training, we consider the exponential loss function to obtain a global optimal solution and avoid model collapse. Furthermore, we use DIRE-Score that calculates both reconstruction loss and discrimination loss to identify abnormal data in the anomaly detection process.
\begin{figure*}[t]
	\centering
	\vspace*{0.5em}
	\includegraphics[width=0.9\linewidth]{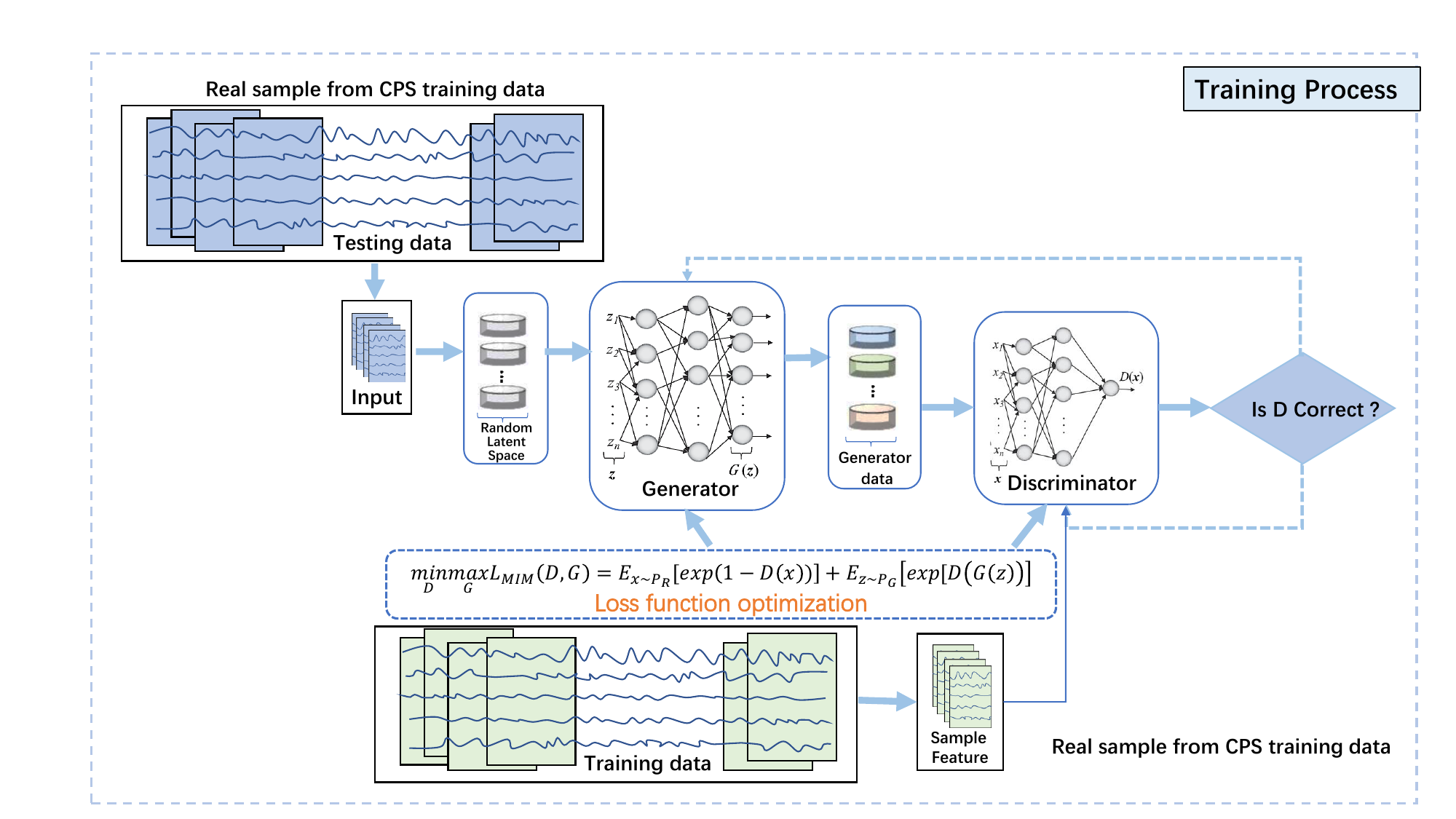}
	\vspace*{-0.5em}
	\caption{Illustration of MIM-GAN training.}\label{abf2222}
\end{figure*}
\subsection{MIM-GAN Training}
As illustrated in Fig. \ref{abf2222}, we employ Long Short Term Memory (LSTM) networks for both G and D in our model. The G takes random time series data and generates synthetic time series sequences. Meanwhile, the D's role is to distinguish between the synthetic time series and the real normal time series from the training data.
During the training process, the D becomes proficient at accurately classifying real and synthetic time series. On the other hand, the G's objective is to deceive the discriminator by generating synthetic time series that closely resemble real ones. After numerous iterations, the D becomes adept at classifying time series into real or fake categories, while the G learns the implicit multivariate distribution of the normal state from the training data. This dynamic interplay between the G and D enables the GAN to effectively learn and generate multivariate time series data with realistic characteristics.

Note that the MIM-GAN model is a GAN that incorporates the information measurement\cite{9405665}, where the logarithmic function of GAN is replaced by an exponential function. In particular, the traditional logarithmic function used in standard GANs is replaced with an exponential function. This alteration introduces potential advantages, particularly due to the amplification effect observed in MIM-GAN.

The replacement of the logarithmic function with the exponential function in the objective function leads to the concept of message identification (M-I) divergence. The exponential function's sensitivity allows MIM-GAN to discern information differences more effectively. As a result, it exhibits distinct characteristics in information representation compared to logarithmic functions.

It is worth emphasizing that the exponential function is convex, which can contribute to improved optimization properties and potentially lead to more stable training dynamics in the MIM-GAN model. We can convert the maxmin optimization problem into the equivalent minmax optimization problem:
\begin{align}\label{minmaxa12}
\min\limits_{\mathrm{D}}\space\max\limits_{\mathrm{G}}\space \mathcal{L}_{{MIM\bf}}(\mathrm{D},\mathrm{G}),
\end{align}
where the loss function $\mathcal{L}(D, G)$ is defined as
\begin{align}\label{sldoe4}
\!\!\!\mathcal{L}_{\mathrm{MIM}}(D,\! G)
&\!\!\!=\!\!\E_{x \sim P_{R}}[\exp (1\!-\!\!D(x))]\!\!+\!\!\ E_{z \sim P_{G}}[\exp[D(G(z))] \nonumber \\
&\!\!\!\!\!\!\!\!\!\!=\!\E_{x \sim P_{R}}\![\exp (1\!-\!D(x))]\!+\!\!\E_{x \sim P_{G}}[\exp (D(x))],
\end{align}

  Especially, exponential functions and logarithmic functions have different characteristics in information representation. It is important to point out that the the exponential loss function is convex, thus we can obtain a global optimal solution.
\subsubsection{Global Optimal Solution of Loss Function}
A local optimal solution refers to a state where the GAN or GAN-based network falls into a local optimal state during the training process. Then the generator or discriminator is unable to improve their performance, but stays at a suboptimal level.
The training process of GAN is a non-convex optimization problem. This may affect the efficiency of using GAN for anomaly detection, because the goal of anomaly detection is to identify the features of data that are inconsistent with the distribution of normal data.
If GAN training can achieve global optimal solution, then the balance between G and D will be maintained. G is able to generate data that matches the distribution of the real data, while D is able to accurately distinguish real data from generated data. It is benefit to use GAN for anomaly detection, because D can serve as an effective anomaly detector by treating generated data as normal data and identifying abnormal points in real data as anomalous.

The optimal solution of the loss function is obtained under the condition of the optimal discriminator, which is given by the following Lemma.
\begin{lemma}\label{lemmma11}
The optimal discriminator $D_{\text {MIM }}^{*}(x)$ with a given generator G is given by
\begin{align}
&D_{\text {MIM }}^{*}(x)=\frac{1}{2}+\frac{1}{2} \ln \left\{\frac{P_{R}(x)}{P_{G}(x)}\right\},
\end{align}
where $P_{R}(x)$ represents the distribution of real data, $P_{G}(x)$ represents the distribution of generated data.
\end{lemma}
\begin{proof}
The loss function in \eqref{sldoe4} can be formulated as
\begin{align}
\mathcal{L}_{MIM}(D,G)=\int_{\mathcal{X}}[P_{R}(x)e^{1-D(x)}+P_{G}(x)e^{D(x)}]dx,
\end{align}
where $\mathcal{X}$ is a domain for $x$. Let $a=P_R(x)>0$, $b=P_G(x)>0$, and $u=D(x)$, then we define a function $F(u)$ as
\begin{align}
F(u)=a\ \mathrm{exp}(1-u)+b\ \mathrm{exp}(u),
\end{align}
Note that $F(u)$ is a continuously differentiate function, and its second derivative satisfies
\begin{align}
\frac{\partial^2(F(u))}{\partial u^2}=a\ \mathrm{exp}(1-u)+b\ \mathrm{exp}(u)>0,
\end{align}
 This means that $F(u)$ is convex with respect to $u$. Let the first derivative be equal to zero, i.e.,
\begin{align}
\frac{\partial F(u)}{\partial u}=-a\ \mathrm{exp}(1-u)+b\ \mathrm{exp}(u)=0,
\end{align}
we obtain the optimal solution $u=\frac{1}{2}+\frac{1}{2}\ln(\frac{a}{b})$. Since $a>0$ and $b>0$, $\frac{1}{2}+\frac{1}{2}\ln(\frac{a}{b})$ is the minimum value of function $F(u)$. As $u=D(x)$, the optimal solution of discriminator is $D_{MIM}^\ast\left(x\right)=\frac{1}{2}+\frac{1}{2}\ln{\left\{\frac{P_R\left(x\right)} {P_G \left(x \right)} \right \}}$.
 \end{proof}
With the optimal solution $D_{\mathrm{MIM}}^{*}(x)$ in Lemma \ref{lemmma11} and a generator G, then the loss function $\mathcal{L}_{\mathrm{MIM}}(D, G)$ in \eqref{minmaxa12} can be written as
\begin{align}\label{sldp2}
&\mathcal{L}_{\mathrm{MIM}}\left(D=D_{\mathrm{MIM}}^{*}, G\right)=\mathbb{E}_{x\sim P_{R}}\left[\exp \left(1-D_{\mathrm{MIM}}^{*}(x)\right)\right]\nonumber \\
&+\mathbb{E}_{x\sim P_{G}}\left[\exp \left(D_{\mathrm{MIM}}^{*}(x)\right)\right]\nonumber \\
&=\mathbb{E}_{x\sim P_R}\left[\exp \left(\frac{1}{2}+\ln \left(\frac{P_{R}(x)}{P_{\mathrm{G}}(x)}\right)^{-\frac{1}{2}}\right)\right]\nonumber \\
&+\mathbb{E}_{x \sim P_{G}}\left[\exp \left(\frac{1}{2}+\ln \left(\frac{P_{R}(x)}{P_{\mathrm{G}}(x)}\right)^{\frac{1}{2}}\right)\right]\nonumber \\
\!\!\!\!\!\!\!&=\sqrt{\mathrm{e}}\left\{\mathbb{E}_{x\sim P_R}\!\!\!\left[\left(\frac{P_{R}(x)}{P_{\mathrm{G}}(x)}\right)^{-\frac{1}{2}}\right]\!+\!\!\mathbb{E}_{x \sim P_{G}}\!\!\left[\left(\frac{P_{\mathrm{G}}(x)}{P_{R}(x)}\right)^{-\frac{1}{2}}\right]\right\},
\end{align}

\begin{proposition}
With the optimal discriminator $D_{\text {MIM }}^{*}(x)$, the optimal solution of the loss function $\mathcal{L}_{MIM}(D=D_{\mathrm{MIM}}^{*},G)$ is given by
\begin{align}\label{propose}
\mathcal{L}^{*}_{MIM}(D=D_{\mathrm{MIM}}^{*},G)=2\sqrt{e},
\end{align}
\end{proposition}

\begin{proof}
From \eqref{sldp2}, we have
\begin{align}\label{po02}
&\max _{G} \mathcal{L}_{\mathrm{MIM}}\left(D=D_{\mathrm{MIM}}^{*}, G\right) \nonumber\\
&\Leftrightarrow \min _{G} \frac{\sqrt{\mathrm{e}}}{2}\left\{R_{\alpha=\frac{1}{2}}\left(P_{R} \| P_{G}\right)+R_{\alpha=\frac{1}{2}}\left(P_{G} \| P_{R}\right)\right\},
\end{align}
where $R_\alpha(\cdot)$ is the Renyi divergence with $  \alpha = \frac{1}{2} $, which is defined as
\begin{align}\label{lspde}
R_\alpha(P_{R}\parallel P_{G})=\frac{1}{\alpha-1}\ln\left(\mathbb{E}_{\mathbf{x}\sim P}\left[{\left(\frac{P(\mathbf{x})}{P_{g_\mu}(\mathbf{x})}\right)}^{\alpha-1}\right]\right),
\end{align}
with $\alpha>0,\alpha\neq 1$.
Based on \eqref{po02} and \eqref{lspde}, we obtain the global maximum of $\mathcal{L}_{MIM}\ (D=D_{MIM}^\ast,G)$ as \eqref{propose}.
\end{proof}

\subsubsection{Resistance to Model Collapse}



The original loss function of GAN makes it difficult to train,  which leads to low diversity of generated samples in most cases. The traditional GANs use binary cross-entropy loss, which can cause the generator gradient to vanish and require careful tuning of the generator and discriminator parameters.
For $\mathrm{GAN}$, the optimal D is $D_{\mathrm{KL}}^{*}(x)=\frac{P_{\mathrm{R}}(x)}{P_{\mathrm{R}}(x)+P_{\mathrm{G}}(x)}$,
the optimization problem for the generator G is
\begin{align}
&\min _{G} \mathbb{E}_{z \sim P_{Z}}\left[\ln \left(1-D_{\mathrm{KL}}^{*}(G(z))\right)\right]\nonumber\\
&\Leftrightarrow \min _{G} \mathbb{E}_{x \sim P_{G}}\left[\ln \left(\frac{P_{\mathrm{G}}(x)}{P_{R}(x)+P_{G}(x)}\right)\right]\nonumber\\
&\Leftrightarrow\min _{G} \int_{\mathcal{X}} P_{G}(x)\left[\ln \left(\frac{P_{G}(x)}{P_{R}(x)}\right)\right] \mathrm{d} x,
\end{align}

The above representative optimization problem is equivalent to minimize the KL distance between $P_{G}$ and $P_{R}$. Therefore, we can observe that the generator of GAN network may not generate the true distribution of the temporal data for $\mathrm{P}_{\mathrm{G}}(x)\ \rightarrow\ 0$.

For MIM-GAN network, the optimization problem for the generator G is
\begin{align}
&\max _{G} \mathbb{E}_{Z \sim P_{Z}}\left[\exp \left(D_{\text {MIM }}^{*}(G(Z))\right)\right] \nonumber\\
&\Leftrightarrow \max _{G} \mathbb{E}_{Z \sim P_{Z}}\left[\operatorname { exp } \left(D_{\mathrm{MIM}}^{*}(x)\right.\right] \nonumber\\
&\Leftrightarrow \max _{G} \mathbb{E}_{Z \sim P_{Z}}\left[\exp \left(\frac{1}{2}+\frac{1}{2} \ln \frac{P_{R}(x)}{P_{\mathrm{g}}(x)}\right)\right] \nonumber\\
&\Leftrightarrow \max _{G} \sqrt{\mathrm{e}}\left\{\mathbb{E}_{x \sim P_{G}}\left[\left(\frac{P_{R}(x)}{P_{G}(x)}\right)^{\frac{3}{2}}\right]\right\} \nonumber\\
&\Leftrightarrow \max _{G} \int_{\mathcal{X}} P_{G}(x)\left(\frac{P_{R}(x)}{P_{G}(x)}\right)^{\frac{1}{2}} \mathrm{~d} x \nonumber\\
&\Leftrightarrow \max _{G} \int_{\mathcal{X}} P_{G}(x)^{\frac{1}{2}} P_{R}(x)^{\frac{1}{2}} \mathrm{~d} x,
\end{align}

When we optimize the generator, it is inevitable for the original GAN to cause that $\mathrm{P}_{G}(x)$ tends to be 0, which automatically results in the mode collapse problem. Based on \eqref{minmaxa12} and \eqref{sldoe4}, we convert the maxmin optimization problem into the equivalent minmax optimization problem. Therefore, when we maximize the generator, the loss function of MIM-GAN is set so that $\mathrm{P}_{G}(x)\ \rightarrow\ 0$ is avoided, thus it doesn't result in the model collapse problem.

\begin{algorithm}[tp]
     \caption{ Proposed MIM-GAN-based Anomaly Detection for Multivarite Time Serise Data }
      \label{Algorithm 1}
     {\bf{Input}}: {{Real Data}} $x$ and $X^{test}$, random latent space $z$ \\
     {\bf Loop}:\\
       1. \hspace*{0.15in} {\bf If} epoch within number of traning interations {\bf then}\\
          \hspace*{0.1in} {\bf For} the $N^{th}$ epoch do:\\
       2. \hspace*{0.1in} Generate samples from the random space:\\
        \hspace*{0.8in}$Z=\{z_i,i=1,\cdots,m \} \Rightarrow G_{\mathrm{lstm}}(Z)$\\
       3. \hspace*{0.1in} Calculate discrimination loss:\\
        \hspace*{0.8in}$X=\{x_i,i=1,\cdots,m\} \Rightarrow D_{\mathrm{lstm}}(X)$\\
       \hspace*{0.8in}$G_{\text{lstm}}(Z)\Rightarrow D_{\mathrm{lstm}}({G}_{{\mathrm{lstm}}}(Z))$\\
       4. {\bf Update D}: minimizing $D_{{loss}}$:\\
       \hspace*{0.5in} ${\min}\hspace*{0.05in}\frac{1}{m}\sum\limits_{i=1}^{m}[\{{\mathrm {exp}}(1-{{D}}_{{\mathrm{lstm}}}(x_i))]+{{\mathrm{exp}}}[{{D}}(G(z_i))]$\\
       5. {\bf Update generator}: Generator parameters by maximizing $G_{loss}$:\\
       $\hspace*{0.8in}\min\hspace*{0.05in}\frac{1}{m}\sum\limits_{i=1}^{m}{\mathrm{exp}}[D_{\mathrm{lstm}}(G_{\mathrm{lstm}}(z_i))]$\\
       6.\hspace*{0.15in} Record parameters of the D and G \\
       7. {\bf End for} \\
       8. \hspace*{0.15in} {\bf End if} \\
       9. For the $L$-th iteration do\\
          \hspace*{0.8in}$Z^k=\min\limits_{z}\hspace*{0.05in}{\mathrm{Err}}(X^{\text{test}},G_{\mathrm{lstm}}(Z^i))$\\
       10. End for\\
       11. \hspace*{0.1in} Caculate the reconstration score:\\
       $\hspace*{0.6in} \mathrm{Rec-Score(X^{\text{test}})= \sum_{i=1}^n |X^{\text{test}} - G_{\text{lstm}}(Z^{k})|}$ \\
       12. \hspace*{0.1in} Caculate the discrimination score:\\
       $\hspace*{0.6in} \mathrm{Dis-Score(X^{\text{test}})= D_{lstm}(X^{\text{test}})} $\\
       13. {\bf Obtain} the Anomaly Detection Loss:\\
        \hspace*{0.2in}$\text{AD-Loss}=\alpha\text{Rec-Score}(X^{\text{test}})+\beta\text{Dis-Score}(X^{\text{test}})$; \\
       14. {\bf For} $j,s$ and $t$ in ranges do \\
       15.\hspace*{0.15in} {\bf If} $j+s=t$ {\bf then}\\
       16.\hspace*{0.3in} $\text{DIRE-Score}=\frac{\sum_{j, s \in j+s=t} L_{j, s}}{\operatorname{coun} t(j, s \in j+s=t)}$ \\
       17.\hspace*{0.15in} {\bf  End if}\\
       18. {\bf  End for}\\
        {\bf End loop}\\
{\bf output}: AD-Loss, DIRE-Score
\end{algorithm}

\begin{figure*}[tp]
	\centering
	\vspace*{0.5em}
	\includegraphics[width=0.93\linewidth]{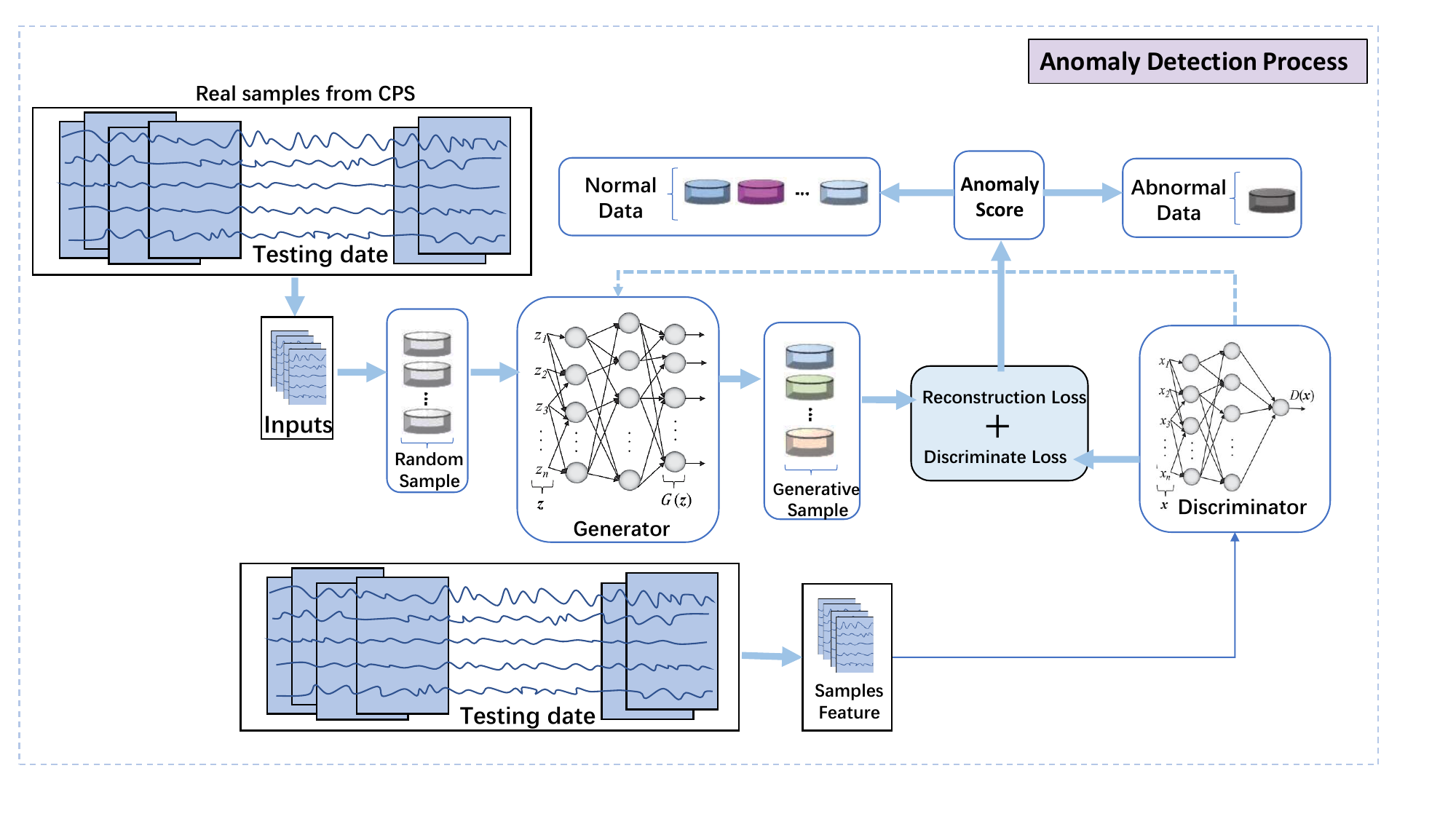}
	\vspace*{-0.5em}
	\caption{The proposed MIM-GAN-based anomaly detection for time series data.}\label{abfl45}
\end{figure*}
\subsection{MIM-GAN-based Anomaly Detection}


Fig. 3. shows the anomaly detection process of our presented an anomaly detection algorithm based on MIM-GAN. We first split the testing data into sub-sequences using a sliding window, and input the time series data into the trained G and D. Then we make use of the trained model to learn the features of the test data. In addition, we input each sub-sequence into the trained generator and discriminator and compute DIRE-Score based on the D-loss and G-loss. Finally, we label each sub-sequence as normal or abnormal based on a predefined threshold.

\subsubsection{Data Processing}

We utilize a large amount of historical or updated data collected from the real world as the input for training the GAN model. The discriminator and generator in the anomaly detection use the model parameters of the previously trained G and D. To perform anomaly detection on a testing dataset, we also divide the testing dataset into multivariate sub-sequences using a sliding window technique.
\subsubsection{Anomaly Score}
For anomaly detection, we leverage both the discriminator's discrimination score and the generator's reconstruction score. The discrimination error quantifies the discriminator's ability to classify the sub-sequence as real or fake, while the reconstruction loss assesses how well the generator can reconstruct the sub-sequence from a random input. By combining these two scores, we obtain a comprehensive metric that indicates the likelihood of the sub-sequence being anomalous. We assign a label to each sub-sequence in the test data as

\begin{align}
A^{\text {test}}=\left\{\begin{array}{lc}1, &\text { if } \textit{T}\left(\mathrm{DIRE-Score}, 1\right)>\tau, \\0, & \!\!\!\!\!\!\!\text { else },\end{array}\right.
\end{align}
where $\textit{T}\left(\mathrm{DIRE-Score}, 1\right)$ represents the cross-entropy loss, $A^{\text {test}}$ is a label vector for the test data as normal or abnormal, and it depends on whether the anomaly score is higher or lower than a threshold. We process the testing data through sliding windows to identify abnormal instances. Subsequently, these abnormal instances are assigned a label of 1 if their anomaly score surpasses the threshold. Similarly, testing data processed by sliding windows is labeled as normal and assigned a label of 0 if its anomaly score falls below the threshold.
The trained generator can generate sufficiently realistic data samples. If corresponding mappings can be found in the latent space of test set $D_{test}$, then the similarity between test set $D_{test}$ and reconstructed samples  $G(z)$  by generator reflects whether test set conforms to the data distribution reflected by generator to some extent. Therefore, the reconstruction residuals between test set and generated samples for identifying anomalies of test data can be expressed as
\begin{align}
\!\!\!\!\min_{Z^k} \text{Err}(X^{\text{test}}, G_{\text{lstm}}(Z^k))\!=\! 1\!-\!\text{Simi}(X^{\text{test}}, G_{\text{lstm}}(Z^k)),
\end{align}
where $X^{\text{test}}$ represents the testing data, $Z^k$ represents the random samples. In addition, the function $\text{Simi}(\cdot)$ are defined as
\begin{align}
\!\!\!\text { Simi }=\frac{A \cdot B}{\|A\|\|B\|}=\frac{\sum_{i=1}^{n} A_{i} \times B_{i}}{\sqrt{\sum_{i=1}^{n}\left(A_{i}\right)^{2}}\times \sqrt{\sum_{i=1}^{n}\left(B_{i}\right)^{2}}},
\end{align}
After enough iterations to achieve sufficiently small errors, $Z^k$ is recorded as a mapping in the latent time series data for its corresponding test sample. The residual calculation of test samples at time $t$ is
\begin{align}
\text{Rec-Score}(X^{\text{test}}) &= \sum_{i=1}^n |X^{\text{test}} - G_{\text{lstm}}(Z_t^{k})|,
\end{align}
where $X^{\text{test}}\subseteq R^n$ are $n$ variables measured. The anomaly detection loss (AD-Loss) is
\begin{align}
\!\!\!\text{AD-Loss} &= \!\!\alpha \text{Rec-Score}(X^{\text{test}}) \!+\! \beta \text{Dis-Score}(X^{\text{test}}),
\end{align}
where $\alpha+\beta=1, \alpha>0, \beta>0$.

Subsequently, the discriminator generates a score, termed ``Dis-Score", and the generator generates a score, known as ``Rec-Score", both of which are utilized to compute a set of anomaly detection losses.
We calculate a DIRE-Score by mapping the losses of anomaly detection for subsequence's back to original time series, which is defined as
\begin{align}
\text{DIRE-Score}=\frac{\sum_{j, s \in j+s=t} L_{j, s}}{\operatorname{coun} t(j, s \in j+s=t)},
\end{align}
where $t\in{1,2,...,N}$, $j\in{1,2,...,n}$, and $s\in{1,2,...,S_w}$.

The proposed MIM-GAN-based anomaly detection is summarized in Algorithm 1.
We use mini-batch stochastic optimization based on AdamW Optimizer and Gradient Descent Optimizer to update model parameters in this work.

\section{Experiment RESULTS}
In this section, we carry out experiments on the KDDCUP99 dataset and compare it with traditional machine learning algorithms including popular unsupervised anomaly detection algorithms. Furthermore, we investigate the impact of different subsequence lengths on model performance. Specifically, we set the subsequence lengths to 30, 60, 90, 120 and 150 to assess the stability of MIM-GAN-based anomaly detection for multivariate time series data.
The precision (Pre), recall (Rec) and F1-score for evaluating the anomaly detection performance of MIM-GAN. Specially, the Pre, Rec, and F1-score are defined as
\begin{align}
\!\!\mathrm{Pre} = \frac{\textit{TP}}{\textit{TP} + \textit{FP}},
\mathrm{Rec} = \frac{\textit{TP}}{\textit{TP} + \textit{FN}},
\mathrm{F1} = 2\cdot \frac{\textit{Pre} \cdot \textit{Rec}}{\textit{Pre} + \textit{Rec}},
\end{align}
where $\textit{TP}$, $\textit{FP}$ and $\textit{FN}$ are true positive, false positive, and false negative, respectively.
The experimental results, as shown in Table 1, demonstrate a significant improvement in our model compared to traditional machine learning methods such as Principal Component Analysis (PCA)\cite{8802933}, K-Nearest Neighbors (KNN)\cite{3880288}, Autoencoder (AE)\cite{8802301}, Local Outlier Factor (LOF)\cite{8594897},  Isolation Forest (IF)\cite{Liu08isolationforest}, One-Class Support Vector Machine (ONE-SVM)\cite{9064715}, Graph Deviation Network (GDN)\cite{deng2021graph}, Deep autoencoding Gaussian Mixture Model (DAGMM)\cite{2018Deep}, Multivariate Anomaly Detection with Generative Adversarial Networks (MAD-GAN)\cite{2306789}. The optimal parameters of the model are as follows: the subsequence length (Seq-length) is 90, the batch size is 512 in the training of all neural networks. The network optimizer is chosen as AdamW, the initial learning rate ($\eta$) is 0.0005. Compared with unsupervised anomaly detection algorithm MAD-GAN, our algorithm has obvious advantages in terms of precision, recall and F1-score. The accuracy is boosted by 5 percent, recall by 1.8 percent, and F1-score by 1 percent. Interestingly, the precision, recall rate, and F1-score values of the detection results exhibit relatively uniform performance.

As illustrated in Fig. \ref{abfl45}, we present a comparison of the generated samples at different training stages. In the early stages, the samples generated by GAN appear to be highly random. However, as the training progresses, the generator gradually adopts the distribution of the original samples, resulting in the generated samples closely resembling the actual data. Towards the later stages of training, the generated samples achieve a significantly improved level of similarity to the original data, almost approaching perfection.

To ensure accurate prediction accuracy and recall rate, as well as meet the practical detection requirements, an appropriate lengths of subsequence is needed. It is evident that the performance is relatively favorable during the initial iterations, particularly when the lengths of subsequence are set as 90 from Fig. \ref{abfl}.
The subsequence length is too small, it could lead to low temporal correlation of the window sequence, losing a lot of information on temporal relationship.\\
\begin{table}
\caption{Metrics OF ANOMALY DETECTION ON KDDCUP99}\label{psnr1}
\centering
\begin{tabular}{cccccc}
\cline{1-5}
\multirow{2}*{}&\multirow{2}*{Algorithms}&\multirow{2}*{Precision}&\multirow{2}*{Recall}&\multirow{2}*{F1-Score}\\
\multicolumn{2}{c}{} \\
\cline{1-5}
\multirow{11}*{}
&FB&48.98&19.36&0.28\\
&PCA&60.66&37.69&0.47\\
&KNN&45.51&18.98&0.53\\
&AE&$\mathbf{80.59}$&42.36&0.55\\
&LOF&75.55&52.54&$\mathbf{0.63}$\\
&IF&78.01&70.16&$\mathbf{0.74}$\\
&ONE-SVM&80.72&73.38&$\mathbf{0.77}$\\
&GDN&85.91&85.20&$\mathbf{0.85}$\\
&DAGMM&90.71&80.52&$\mathbf{0.86}$\\
&MAD-GAN*&$\mathbf{94.92}$&19.14&0.32\\
&MAD-GAN**&81.58&$\mathbf{94.92}$&0.88\\
&MAD-GAN***&86.91&94.79&$\mathbf{0.90}$\\
&$\mathbf{MIM}$-$\mathbf{GAN}$*&$\mathbf{99.89}$&2.3&0.05\\
&$\mathbf{MIM}$-$\mathbf{GAN}$**&83.44&$\mathbf{96.73}$&0.90\\
&$\mathbf{MIM}$-$\mathbf{GAN}$***&95.81&86.71&$\mathbf{0.91}$\\
\cline{1-5}
\end{tabular}
\end{table}
\begin{table}
\caption{Experiment Parameters}\label{psnr1}
\centering
\begin{tabular}{cccccc}
\cline{1-5}
\multirow{2}*{}&\multirow{2}*{Algorithms}& &\multirow{2}*{Parameters}\\
\multicolumn{2}{c}{} \\
\cline{1-5}
\multirow{11}*{}
&$\mathbf{MIM}$-$\mathbf{GAN}$*&Seq-length=30&batch-size=256&$\eta$=0.0001\\
&$\mathbf{MIM}$-$\mathbf{GAN}$**&Seq-length=60&batch-size=256&$\eta$=0.0001\\
&$\mathbf{MIM}$-$\mathbf{GAN}$***&Seq-length=90&batch-size=512&$\eta$=0.0005\\
\cline{1-5}
\end{tabular}
\end{table}
\section{CONCLUSION}
\vspace{0.6em}
\begin{figure*}[tp]
	\centering
	\vspace*{0.5em}
	\includegraphics[width=0.8\linewidth]{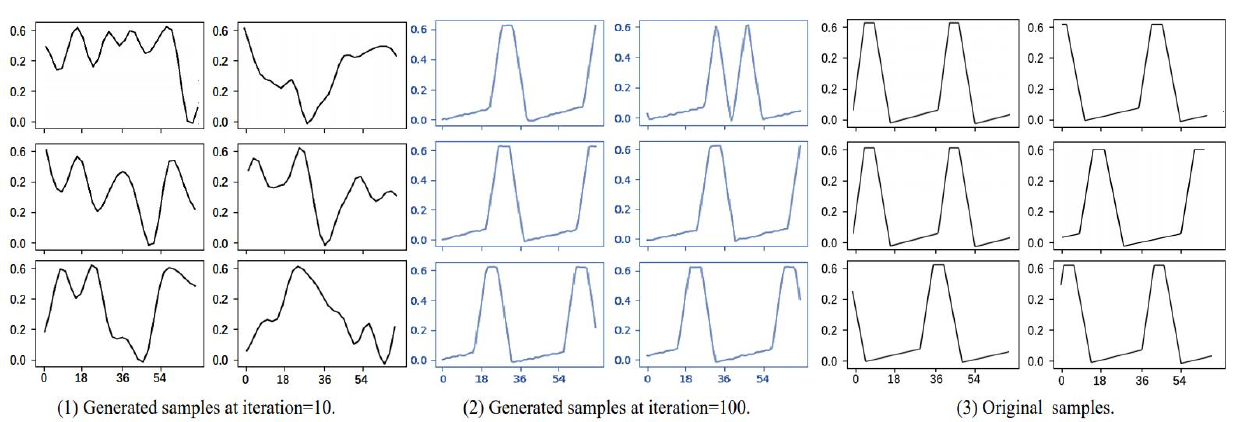}
	\vspace*{-0.5em}
	\caption{Comparison of the generated samples at different training stages.}\label{abfl45}
\end{figure*}
\begin{figure}[tp]
	\centering
	\vspace*{0.5em}
	\includegraphics[width=0.639\linewidth]{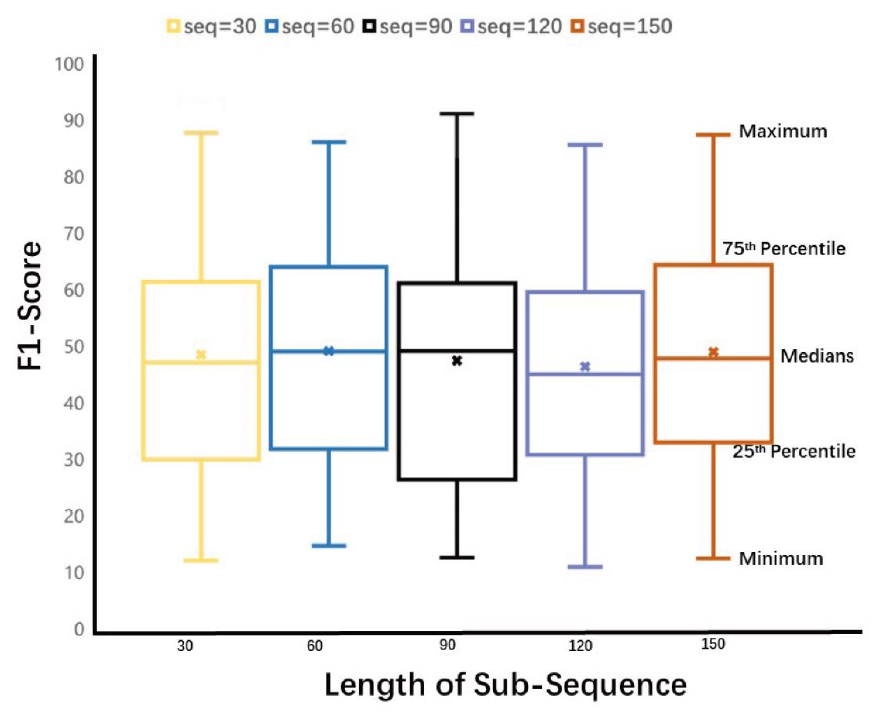}
	\vspace*{-0.5em}
	\caption{F1-Score of MIM-GAN-based anomaly detection with different lengths of sub-sequences Length.}\label{abfl}
\end{figure}
In this paper, we adopted an exponential form of loss function by referring to Message Importance Measure (MIM), we used MIM-GAN algorithm for anomaly detection of multivariate time series data, including the MIM-GAN training process and anomaly detection process. Specially, we used the sliding window to preprocess the time series data and exploited the temporal and spatial correlations of the time series by the LSTM , which captures the temporal relationship and the dependency among variables. Furthermore, we theoretically analyzed the advantages of MIM-GAN in terms of global optimal solution of loss function and the model collapse. The conducted experiments shown the excellent precision, recall, and F1-score of the proposed MIM-GAN-based anomaly detection algorithm.
\section{ACKNOWLEDGMENT}
 \vspace{0.6em}
 This work was supported by Science and Technology Major Project of Tibetan Autonomous Region of China (XZ202201ZD0006G04), National Key Research and Development Program of China (Grant No. 2020YFC0833406), Key Research and Development and Transformation Plan of Science and Technology Program for Tibet Autonomous Region China (No. XZ201901-GB-16), Famous teachers workshop project of Tibet University China, Central government China supports the reform and development of local universities of Tibet University in 2020. The code is available at https://github.com/ExplorerLu1024/MIMAD-GAN

\vspace{0.6em}
\bibliographystyle{IEEEtran}
\bibliography{IEEEfull,trans}
\vspace{0.6em}

\end{document}